\documentclass[11pt,letterpaper]{article}

\newif\ifnips\nipsfalse
\ifnips
\usepackage{nips_2018}
\else
\usepackage[margin=1in]{geometry}
\fi

\renewcommand{\phi}{\varphi}

\newcommand{\R}{\mathbb{R}}

\newcommand{\cD}{\mathcal{D}}

\newcommand{\cN}{\mathcal{N}}

\newcommand{\cF}{\mathcal{F}}

\newcommand{\cX}{\mathcal{X}}

\def\ds1{\mathds{1}}
\renewcommand{\epsilon}{\varepsilon}
\newcommand{\eps}{\epsilon}

\newcommand{\wh}{\widehat}
\newcommand{\wt}{\widetilde}

\newcommand{\argmin}{\mathop{\mathrm{arg\,min}}}

\newlength{\minipagewidth}
\newcommand{\beq}{\begin{equation}}
\newcommand{\eeq}{\end{equation}}

\newcommand{\beqa}{\begin{eqnarray}}
\newcommand{\eeqa}{\end{eqnarray}}

\newcommand{\beqan}{\begin{eqnarray*}}
\newcommand{\eeqan}{\end{eqnarray*}}

\def\ba#1\ea{\begin{align*}#1\end{align*}} 
\def\banum#1\eanum{\begin{align}#1\end{align}} 

\newcommand{\D}{\mathcal{D}}

\newcommand{\norm}[1]{\|#1\|}
\newcommand{\FF}{\mathcal{F}}

\newcommand{\Rbb}{\mathbb{R}}

\newcommand{\Uc}{\mathcal{U}}
\newcommand{\proj}{\mathrm{proj}}

\usepackage{etoolbox}

\makeatletter

\newcommand{\define}[4][ignore]{%
  \ifstrequal{#1}{ignore}{}{
  \@namedef{thmtitle@#2}{#1}}%
  \@namedef{thm@#2}{#4}%
  \@namedef{thmtypen@#2}{lemma}%
  \newtheorem{thmtype@#2}[theorem]{#3}%
  \newtheorem*{thmtypealt@#2}{#3~\ref{#2}}%
}

\newcommand{\state}[1]{%
  \@namedef{curthm}{#1}
  \@ifundefined{thmtitle@#1}{
  \begin{thmtype@#1}
    }{
  \begin{thmtype@#1}[\@nameuse{thmtitle@#1}]
  }
    \label{#1}
    \@nameuse{thm@#1}
  \end{thmtype@#1}
  \@ifundefined{thmdone@#1}{
  \@namedef{thmdone@#1}{stated}%
  }{}
}

\newcommand{\restate}[1]{%
  \@namedef{curthm}{#1}
  \@ifundefined{thmtitle@#1}{
    \begin{thmtypealt@#1}
    }{
  \begin{thmtypealt@#1}[\@nameuse{thmtitle@#1}]
  }
    \@nameuse{thm@#1}
  \end{thmtypealt@#1}
  \@ifundefined{thmdone@#1}{
  \@namedef{thmdone@#1}{stated}%
  }{}
}

\newcommand{\thmlabel}[1]{
  \@ifundefined{thmdone@\@nameuse{curthm}}{\label{#1}
    }{\tag*{\eqref{#1}}}
}
\makeatother

\usepackage[hidelinks]{hyperref}
\usepackage{url}
\usepackage[numbers]{natbib}
\usepackage{graphicx,subfigure,amsmath,amssymb,amsfonts,bm,epsfig,epsf,url,dsfont,color}

\usepackage{amsthm}
\newtheorem{theorem}{Theorem}[section]
\newtheorem{lemma}[theorem]{Lemma}
\newtheorem{definition}[theorem]{Definition}
\newtheorem{remark}[theorem]{Remark}

\DeclareMathOperator*{\E}{\mathbb{E}}
\let\P\undefined
\DeclareMathOperator*{\P}{\mathbb{P}}

\title{Adversarial examples from computational constraints}

\ifnips\else\def\And{\and}\fi
\author{S\'ebastien Bubeck\\
       Microsoft Research\\
			 \texttt{sebubeck@microsoft.com}\\
			 \And
       Eric Price\\
       UT Austin\\
			 \texttt{ecprice@cs.utexas.edu}\\
			 \And
      Ilya Razenshteyn\\
       Microsoft Research\\
			 \texttt{ilyaraz@microsoft.com}\\
}

\begin{document}

\maketitle

\begin{abstract}
Why are classifiers in high dimension vulnerable to ``adversarial" perturbations?
We show that it is likely not due to information theoretic limitations, but rather it could be due to computational constraints.

First we prove that, for a broad set of classification tasks, the mere existence of a robust classifier implies that it can be found by a possibly exponential-time algorithm with relatively few training examples.
Then we give
a particular classification task where learning a
  robust classifier is computationally intractable.  More precisely we
  construct a binary classification task in high dimensional space
  which is (i) information theoretically easy to learn robustly for
  large perturbations, (ii) efficiently learnable (non-robustly) by a
  simple linear separator, (iii) yet is not efficiently robustly
  learnable, even for small perturbations, by any algorithm in the
  statistical query (SQ) model. This example gives an
  exponential separation between classical learning and robust
  learning in the statistical query model. It suggests that
  adversarial examples may be an unavoidable byproduct of
  computational limitations of learning algorithms.
\end{abstract}

\section{Introduction}
The most basic task in learning theory is to learn from a data set $(X_i, f(X_i))_{i \in [n]}$ a good approximation to the unknown input-output function $f$. One is typically interested in finding a hypothesis function $h$ with small out of sample probability of error. That is, assuming the $X_i$'s are i.i.d.\ from some distribution $D$, one wishes to approximately minimize $\P_{X \sim D}(h(X) \neq f(X))$. A more challenging task is to learn a {\em robust} hypothesis, that is one that would minimize the probability of error against {\em adversarially corrupted examples}. More precisely, assume that the input space is endowed with a norm $\|\cdot\|$ and let $\epsilon>0$ be a fixed robustness parameter. In robust learning the goal is to find $h$ to minimize:
\[
\P_{X \sim D}( \exists \ z \text{ such that } \|z\| \leq \epsilon, \text{ and } h(X+z) \neq f(X+z)) \,.
\]
Such an input $X+z$ in the above event is colloquially referred to as an {\em adversarial example}\footnote{In the literature one sometimes uses a more stringent definition of adversarial examples, where $X$ and $z$ are in addition required to satisfy $f(X+z)=f(X)$. We ignore this requirement here.}.
%
%

Following \citet{G14} there is a rapidly expanding literature exploring the vulnerability of neural networks to adversarially chosen perturbations. The surprising observation is that, say in vision applications, for most images $X \sim D$ the perturbation can be chosen in a way that is imperceptible to a human yet dramatically changes the output of state-of-the-art neural networks. This is a particularly important issue as these neural networks are currently being deployed in real-world situations. Naturally there is by now a large literature (in fact going back at least to \cite{Dalvi04, GR06}) on {\em attacks} (finding adversarial perturbations) and {\em defenses} (making classifiers robust against certain type of attacks).

While we have a sophisticated theory for the classical goal of
minimizing the non-robust probability of error, our understanding of
the robust scenario is still very rudimentary. At the moment, the
``attackers'' seem to be winning the arms race against the
``defenders'', see e.g., \cite{ACW18}.
We identify four mutually exclusive possibilities for why all known
classification algorithms are vulnerable to adversarial examples:

\begin{enumerate}
\item No robust classifier exists.
\item Identifying a robust classifier requires too much training data.
\item Identifying a robust classifier from limited training data is
  information theoretically possible but computationally intractable.
\item We just have not found the right algorithm yet.
\end{enumerate}

The goal of this paper is to provide two pieces of evidence, one in favor
of hypothesis 3 and one against hypothesis 2.  Our primary result is that
hypothesis 3 is indeed possible: there exist robust classification tasks that are
information theoretically easy but computationally intractable under a
powerful model of computation (namely the statistical query model, see below).
Our secondary
result is evidence against hypothesis 2, showing that if a robust
classifier exists then it can be found with relatively few training
examples under a standard assumption on the data distribution (for
example, that the distribution within each label is close to a
Lipschitz generative model, or is drawn from a finite set of
exponential size).

In Section \ref{sec:relatedworks} we discuss related work on adversarial examples in light of those four hypotheses. In Section \ref{sec:SQmodel} we introduce the model of computation under which we will prove intractability.  We conclude the introduction with Section \ref{sec:overview} where we give a brief proof overview for our primary and secondary result. These results are discussed in greater depth respectively in Section \ref{lower_bound_section} and Section \ref{upper_bound_section}.




\subsection{Related work on adversarial examples} \label{sec:relatedworks}%
To the best of our knowledge, previous works have not linked
computational constraints to adversarial examples, but instead have
focused on the other three hypotheses.

Supporting hypothesis 1 is the work of \citet{FFF18}.  Here the authors
consider a generative model for the features, namely $X = g(r)$ where
$r \in \R^d$ is sampled from an isotropic Gaussian (in particular it
is typically of Euclidean norm roughly $\sqrt{d}$). The observation is
that, due to Gaussian isoperimetry, no classifier is robust to
perturbations in $r$ of Euclidean norm $O(1)$. If $g$ is
$L$-Lipschitz, this corresponds to perturbations of the image $X$ of
at most $O(L)$.
On the other hand, evidence against
hypothesis 1 is the fact that humans seem to be robust classifiers
with low error rate (albeit nonzero error rate, as shown by examples
in~\cite{elsayed2018adversarial}).  This suggests that, to fit real distributions on images,
the Lipschitz parameter $L$ in the data model assumed in \cite{FFF18} may
be prohibitively large.

Another work arguing the inevitability of adversarial examples is
\citet{G18}.  There the authors propose a simple classification task,
namely distinguishing between samples on the unit sphere in high
dimension and samples on a sphere of radius $R$ bounded away from
$1$. They show experimentally that even in such a simple setup,
state-of-the-art neural networks have adversarial examples at most
points.
We note however that this example only applies to specific
classifiers, since it is easy to construct an efficient robust classifier for
the given example (e.g., just use a linear model on the norm of the
features); thus the ``hardness" here only appears for a given network
structure.

Supporting hypothesis 2 is the work of \citet{M18}.  Here the authors
consider a mixture of two separated Gaussians (isotropic, with means
at distance $\Theta(\sqrt{d})$). With such a separation
a single sample is sufficient to learn non-robustly; but to learn a classifier that
is robust to $O(1)$-size perturbations in $\ell_{\infty}$-norm one needs
$\Omega(\sqrt{d})$ samples. This polynomial separation suggests that
avoiding adversarial examples in high dimension requires a lot more
samples than mere learning---but only up to $\sqrt{d}$ samples.
In fact, since their hard instance is essentially a set of $2^d$ possible distributions, our secondary
result gives a black-box algorithm that would produce a robust classifier with $O(d)$ samples.


Finally the large body of work on ``adversarial defense" can be viewed as investigating hypothesis 4.
We note that, at the time of writing, the state of the art defense \citet{Madry18} (according to \cite{ACW18}) is still far from being robust. Indeed on the CIFAR-10 dataset its accuracy is below $50\%$ even with very small perturbations (of order $10^{-2}$ in $\ell_{\infty}$-norm), while state of the art non-robust accuracy is higher than $95\%$.

\subsection{The SQ model} \label{sec:SQmodel}
Proving computational hardness is a notoriously difficult problem. To circumvent this difficulty one usually either (i) reduces the problem at hand to a well-established computational hardness conjecture (e.g., proving NP-hardness), or (ii) proves an unconditional hardness within a limited computational framework (such as the oracle lower bounds in convex optimization, \cite{Nes04}). Our task here is further complicated by the {\em average-case} nature of the problem (the datasets are i.i.d.\ from some fixed distribution). Fortunately there is a growing set of results on computational hardness in learning theory that we can leverage. The statistical query (SQ) model of computation from \citet{Kearns98} is a particularly successful instance of approach (ii) for learning theory: (a) most known learning algorithms fall in the framework, including in particular logistic regression, SVM, stochastic gradient descent, etc; and (b) SQ-hardness has been proved for many interesting problems that are believed to be computationally hard, such as learning parity with noise \cite{Kearns98}, learning intersection of halfspaces \cite{KS}, the planted clique problem \cite{FGRVX}, robust estimation of high-dimensional Gaussians \cite{DKS17}, or learning a function computable by a small neural network~\cite{song2017complexity}.
Thus we naturally use this model to prove our main result on the computational hardness of robust learning. We now recall the definition of the SQ model and state informally our main result.


As Kearns put it in his original paper, the SQ model considers ``learning algorithms that construct a hypothesis based on statistical properties of large samples rather than on the idiosyncrasies of a particular sample". More precisely, rather than having access to a data set $(X_i,f(X_i))$, in the SQ model one must make queries to a $\tau$-SQ oracle which operates as follows: given a $[0,1]$-valued function $\psi$ defined on input/output pairs, the SQ oracle returns a value $\E_{X \sim \cD} \psi(X, f(X)) + \xi$ where $|\xi| \leq \tau$. We refer to $\tau$ as the {\em precision} of the oracle. Obviously, an algorithm using $T$ queries to an oracle with precision $\tau$ can be simulated using a data set of size roughly $T/\tau^2$. In our main result we consider an oracle with {\em exponential} precision. More concretely we take $\tau$ of order $\exp(- C d^c)$ where $d$ is the dimension of the problem and $c, C>0$ are some numerical constants. Observe that such a high precision oracle cannot be simulated with a polynomial (in $d$) number of samples. Yet we show that even with such a high precision one needs an exponential number of queries to achieve robust learning for a certain task which on the other hand is easy to learn, and information theoretically learnable robustly:

\begin{theorem}[informal] \label{th:main}
For any $M, \epsilon > 0$, there exists a classification task in $\R^d$ which is
\begin{itemize}
\item learnable in $\mathrm{poly}(d)$ time and $\mathrm{poly}(d)$ samples;
\item robustly learnable in $\mathrm{poly}(d)$ samples with $\ell_2$-robustness parameter $M$ (while with high probability all samples have $\ell_2$-norm $O(\sqrt{d})$);
\item not efficiently and robustly learnable in the statistical query model, in the sense that even with an exponential (in $d$) precision statistical query oracle one needs an exponential (in $d$) number of queries in order to robustly learn with robustness parameter $\epsilon$.
\end{itemize}
\end{theorem}
The same result holds using the $\ell_\infty$ norm instead of $\ell_2$, except with diameter $O(\sqrt{d \log d})$.


Of course, a number of natural machine learning algorithms such as
nearest neighbor are not based on statistical queries.  Although we
cannot prove it, we believe that our input distributions are
computationally hard in general.  For the case of nearest neighbor,
the distance to points of each class have very similar
distributions---indeed, the two distributions match on polynomially
many moments.  This suggests that exponentially many samples are
necessary for nearest neighbor. For more information about nearest neighbor
classifiers in the context of adversarial examples, see~\cite{wang2017analyzing}.

Moreover, there are very few problems in any domain with exponential
SQ hardness for which polynomial time algorithms are known; in fact,
the only such problems involve solving systems of linear equations
over finite fields~\cite{feldman2017general}.  Since
Theorem~\ref{th:main} involves a real-valued problem, finding a
polynomial time algorithm that avoids the SQ lower bound would be a
remarkable breakthrough in SQ theory.

\subsection{Overview of proofs} \label{sec:overview}

Our secondary result, on the information theoretic achievability of robustness, is proved via
simple arguments reminiscent of PAC-learning theory.
Namely, if a
classifier is not good enough for a given pair of distributions, we
can rule it out with high confidence by looking at not too many
samples. Then, we use a union bound to claim the result for a family
of pairs that is either at most exponentially large, or is at least
covered by a net of at most exponential size (the only subtlety is in
the proper definition of a net in this robust context).

Our primarily result, on the hardness of robustness, is technically much more challenging.
The central object in the proof is a natural high-dimensional
generalization of a construction from~\citet{DKS17}. Roughly speaking,
a hard pair of distributions is obtained by taking a standard
multivariate Gaussian, choosing a random $k$-dimensional subspace and
planting there two well-separated distributions that match many
moments of a Gaussian (in~\cite{DKS17} only the case $k = 1$ is considered).
To show an SQ lower bound, we use -- as in~\cite{DKS17}
-- the framework of~\cite{Blum94,FGRVX} to reduce the question to
computing a certain non-standard notion of correlation between the
distributions. To bound said correlation, we deviate from~\cite{DKS17}
significantly, since their argument is tailored crucially to the case
$k = 1$. Our argument is less precise, but allows $k \gg 1$ which is necessary
to obtain a large separation between the distributions (which in turn controls the parameter $M$
in Theorem \ref{th:main}).

\section{Definitions}

Throughout we
restrict ourselves to binary classifiers, $\R^d$-feature space, as well as to balanced classes. We fix some norm $\|\cdot\|$ in $\R^d$, and we denote $B(\epsilon) = \{z \in \R^d : \|z\| \leq \epsilon\}$.

\begin{definition}
The $\epsilon$-robust zero-one loss (with respect to $\|\cdot\|$) is defined as follows, for $f : \R^d \rightarrow \{0,1\}$ and $(x,i) \in \R^d \times \{0,1\}$,
\[
\ell_{\epsilon}(f, x, i) = \ds1\{\exists \ z \in B(\epsilon) : f(x+z) \neq i\} \,.
\]
\end{definition}

\begin{definition}
A binary classifier $f: \R^d \to \{0, 1\}$  is $(\epsilon, \delta)$-robust for a pair of distributions $(D_0, D_1)$ on $\mathcal{X}$ if for any $i \in \{0,1\}$,
\[
\E_{X \sim D_i} [ \ell_{\epsilon}(f, X, i) ] \leq \delta \,.
\]
\end{definition}

\begin{definition}
  A (binary) \emph{classification task} is given by a family $\D$ of
  pairs of distributions $D=(D_0,D_1)$ over a domain $\mathcal{X}$.  A
  classification algorithm receives datasets $\underline{X}_0, \underline{X}_1$ consisting of
  $n$ i.i.d. samples from $D_0$ and $D_1$ respectively, and outputs a
  classifier $f: \R^d \to \{0, 1\}$.

 We say that $\D$ is $(\eps, \delta)$-robustly learnable with $n$ samples if there is a classification algorithm such that, for
  every $D \in \D$, with probability at least $2/3$ over $\underline{X}_0$ and $\underline{X}_1$, the algorithm produces a classifier $f$ that is $(\epsilon,\delta)$-robust for $D$.
\end{definition}

\begin{remark}
The success probability $2/3$ is an arbitrary constant larger than $1/2$. It is easy to see that, for any $\eta >0$, by using $O(n \log(1/\eta))$ samples one can obtain a success probability of $1-\eta$.

We also note that the classical $(\eps', \delta')$-PAC learning scenario, with $\delta'=1/3$, corresponds to our definition of $(\eps, \delta)$-robust classification with parameters $\eps=0$ and $\delta=\eps'$. Slightly more precisely, a concept class $\cF \subset \{0,1\}^{\R^d}$ for PAC-learning corresponds to the family $\cD$ of all pairs of distribution supported respectively on $f^{-1}(0)$ and $f^{-1}(1)$ for some $f \in \cF$.
\end{remark}

\begin{definition}
We say that $\D$ is $(\epsilon, \delta)$-robustly feasible if every $D \in \D$ admits an $(\epsilon, \delta)$-robust classifier. When it exists we denote $f_D$ for such a classifier (chosen arbitrarily among all robust classifiers for $D$), and $\FF_{\D} = \{f_D, D \in \D\}$.
\end{definition}

\section{Robust learning with few samples}
\label{upper_bound_section}
Obviously robust feasibility is a necessary condition for robust
learnability. We show that it is in fact sufficient, even for {\em
  sample efficient} robust learnability.  We first do so when a finite set
of classifiers $\FF_\D$ suffices for robust feasibility.

\subsection{Robust empirical risk minimization}

\begin{theorem} \label{thm:finiteUB}
Assume that $\D$ is $(\epsilon,\delta)$-robustly feasible. Then it is $(\epsilon, \delta+\delta')$-robustly learnable with $n=\Omega\left(\frac{\delta + \delta'}{\delta'^2} \log(|\cF_{\D}|)\right)$.
\end{theorem}

\begin{proof}
Let $\hat{D}_i = \frac{1}{n} \sum_{j=1}^n \delta_{\underline{X}_i(j)}$ be the empirical measure corresponding to the dataset $\underline{X}_i$. We will show that ERM on the $\epsilon$-robust loss gives the claimed sample complexity. More precisely we consider the classification algorithm that outputs:
\[
\hat{f} = \argmin_{f \in \FF_{\D}} \max_{i \in \{0,1\}} \E_{X \sim \hat{D}_i} \ell_{\epsilon}(f, X, i) \,.
\]
For shorthand notation we write $p_f = \max_{i \in \{0,1\}} \E_{X \sim D_i} \ell_{\epsilon}(f, X, i)$ and $\hat{p}_f = \max_{i \in \{0,1\}} \E_{X \sim \hat{D}_i} \ell_{\epsilon}(f, X, i)$. In particular we simply want to prove that $p_{\hat{f}} \leq \delta + \delta'$. Note that by definition $p_{f_D} \leq \delta$. A standard Chernoff bound gives that, with probability at least $2/3$, one has for {\em every} $f \in \cF_{\cD}$,
\[
|p_f - \hat{p}_f| = O(\sqrt{p_f \log(|\cF_{\D}|) / n}) \,.
\]
Now observe that for $n \geq 4 \frac{\delta + \delta'}{\delta'^2} \log(|\cF_{\D}|)$ one can has $\sqrt{p_{f_D} \log(|\cF_{\D}|) / n} \leq \delta' /2$
, and thus we obtain with $n=\Omega\left(\frac{\delta + \delta'}{\delta'^2} \log(|\cF_{\D}|)\right)$,
\[
p_{\hat{f}} - \frac{\delta'}{2} \sqrt{\frac{p_{\hat{f}}}{\delta+\delta'}} \leq \hat{p}_{\hat{f}} \leq \hat{p}_{f_D} \leq \delta + \delta' \,.
\]
It now suffices to observe that $s \geq \delta + \delta'$ implies $s - \frac{\delta'}{2} \sqrt{\frac{s}{\delta+\delta'}} > \delta + \frac{\delta'}{2}$.
\end{proof}

\subsection{Robust covering number}
In many natural situations the classification task is specified by a continuous set of distributions. For example one might have a set of the form $\cD = \{(g_0(w_0), g_1(w_1)), (w_0, w_1) \in \Omega\}$ where $g_0$ and $g_1$ are Lipschitz functions and $\Omega$ is some compact subset of $\R^{d'}$. In this case Theorem \ref{thm:finiteUB} does not apply, although one would like to say that ``essentially" $\cD$ is of log-size roughly $d'$. The classical solution to this difficulty is with covering numbers:
\begin{definition}
For a metric space $(\cX, \mathrm{dist})$ we write \[
\cN_{\mathrm{dist}}(\cX, \epsilon) = \inf\left\{ |X| \text{ s.t. } X \subset \cX \text{ and } \cX \subset \bigcup_{x \in X} \{y : \mathrm{dist}(x,y) \leq \epsilon\} \right\} \,.
\]
With a slight abuse of notation we also extend the distance to the Cartesian product $\cX \times \cX$ by $\mathrm{dist}((x,x'), (y,y')) = \max(\mathrm{dist}(x,x'), \mathrm{dist}(y,y'))$.
\end{definition}
With the above definitions one can obtain the following result as a straightforward corollary of Theorem \ref{thm:finiteUB} and the definition of total variation distance.

\begin{theorem} \label{thm:covering1}
Assume that $\D$ is $(\epsilon,\delta)$-robustly feasible. Then $\cD$ is $(\epsilon, \delta+2\delta')$-robustly learnable with $n=\Omega\left(\frac{\delta + \delta'}{\delta'^2} \log(\cN_{\mathrm{TV}}(\D, \delta'))\right)$.
\end{theorem}

In fact, if one is willing to lose a little bit of robustness, one can use a significantly weaker notion of ``distance" than total variation. Indeed we can consider a broader class of modifications to a distribution that preserves the robustness of a classifier: in Theorem \ref{thm:covering1} we used that we can move arbitrarily a small amount of mass, but in fact we can also move {\em a little} an arbitrary amount of mass. While the former type of movement corresponds to total variation distance, the latter corresponds to the (infinity) Wasserstein distance. We denote $W_{\infty}(D,D')$ for the infimum of $\sup_{(x,x') \in \mathrm{supp}(\mu)} \|x-x'\|$ over all measures $\mu(x,x')$ with marginal over $x$ (respectively $x'$) equal to $D$ (respectively $D'$). Next we introduce a slightly non-standard notion of covering with respect to a pair of distances

\begin{definition} \label{def:subtlecover}
For a metric space $\cX$ equipped with two distances $\mathrm{dist}$ and $\mathrm{dist}'$ we define an $(\epsilon, \delta)$  neighborhood by\footnote{The choice of first moving with $\mathrm{dist}'$ and then with $\mathrm{dist}$ will fit our application. In general a more natural definition would be:
\[
U_{\epsilon, \delta}(x) = \{y : \exists x=z_1,z_1', \hdots, z_n, z_n'=y \text{ s.t. } \sum_{i=1}^n \mathrm{dist}(z_i,z_i') \leq \epsilon \text{ and } \sum_{i=1}^{n-1} \mathrm{dist}'(z_i',z_{i+1}) \leq \delta\} \,.
\]}:
\[
U_{\epsilon, \delta}(x) = \left\{y : \exists z \text{ s.t. } \mathrm{dist}'(x,z) \leq \delta \text{ and } \mathrm{dist}(z, y) \leq \epsilon \right\} \,.
\]
The corresponding covering number is:
 \[
\cN_{\mathrm{dist}, \mathrm{dist}'}(\cX, \epsilon, \delta) = \inf\left\{ |X| \text{ s.t. } X \subset \cX \text{ and } \cX \subset \bigcup_{x \in X} U_{\epsilon, \delta}(x) \right\} \,.
\]
\end{definition}

It is now easy to prove the following strengthening of Theorem \ref{thm:covering1}:

\begin{theorem} \label{thm:covering2}
Assume that $\D$ is $(\epsilon,\delta)$-robustly feasible. Then $\cD$ is $(\epsilon - \epsilon', \delta+2\delta')$-robustly learnable with $n=\Omega\left(\frac{\delta + \delta'}{\delta'^2} \log(\cN_{W_{\infty}, \mathrm{TV}}(\D, \epsilon', \delta'))\right)$.
\end{theorem}

\begin{proof}
Let $A$ be the set realizing the infimum in the definition of $\cN_{W_{\infty}, \mathrm{TV}}(\D, \epsilon', \delta')$. Observe that $\cD$ is $(\epsilon-\epsilon', \delta + \delta')$-robustly feasible with classifiers from $\cF_A$, and apply Theorem \ref{thm:finiteUB}.
\end{proof}

\subsection{Covering number bound from generative models}
We now show that distributions approximated by generative models have bounded covering numbers (in terms of Definition \ref{def:subtlecover}), so Theorem~\ref{thm:covering2} gives a good sample complexity for such distributions.  The proof is deferred to Appendix~\ref{app:gencover} in the supplementary material.
\begin{definition}
  A \emph{generative model} $g_w: \R^k \to \R^d$ is a neural network
  indexed by weights $w \in \R^m$.  The \emph{generated distribution}
  $D(g_w)$ is the distribution given by $g_w(x)$ for
  $x \sim N(0, I_k)$.
\end{definition}
\define{lem:gencover}{Lemma}{%
  Let $g_w$ be an $\ell$-layer neural network architecture with at most
  $d$ activations in each layer and Lipschitz nonlinearities such as
  ReLUs.  Consider any family of distribution pairs $\D$ such that for each
  $D \in \D$, and each $i \in \{0, 1\}$, there exists some $w \in [-B, B]^m$ with
  $W_{\infty}(D_i, D(g_w)) \leq \eps$.  Then
  \[
    \log\left(\cN_{W_{\infty}, \mathrm{TV}}(\D, \epsilon + \delta, \delta)\right) \leq O(m \ell \log (dB/\delta)).
  \]
}
\state{lem:gencover}

\section{Lower bound for the SQ model}\label{lower_bound_section}


Let $D_0$ and $D_1$ be two distributions over a set $\mathcal{X}$, for which we would like to solve a (binary)
classification task. The SQ model, introduced in~\cite{Kearns98}, is defined as follows. An algorithm
is allowed to access $D_0$ and $D_1$ through \emph{queries} of the following kind.
A query is specified by a function $h \colon \mathcal{X} \to [0, 1]$, and the response is two
numbers $u, v \in \Rbb$ such that $u \in \mathbb{E}_{x \sim D_0}[h(x)] \pm \tau$
and $v \in \mathbb{E}_{x \sim D_1}[h(x)] \pm \tau$. Here $\tau > 0$ is a positive parameter
called \emph{precision}. After asking a number of such queries, the algorithm must output a required
(robust or non-robust) classifier
for $D_0$ and $D_1$.


Our main result is as follows:

\begin{theorem}
For every sufficiently small $\rho, \gamma > 0$ the following holds.
There exists a family of $2^{d^{O(1)}}$ pairs of distributions $(\widetilde{D}_0, \widetilde{D}_1)$ over $\Rbb^d$ such that:
\begin{itemize}
\item Almost all the mass of $\widetilde{D}_0$ and $\widetilde{D}_1$ is supported in an $\ell_2$-ball of radius $O(\sqrt{d})$;
\item The distributions $\widetilde{D}_0$ and $\widetilde{D}_1$ admits a $(\Omega(\sqrt{1 / \gamma}), 2^{-d^{\Omega(\gamma)}})$-robust classifier; moreover,
a $\Omega(\sqrt{1 / \gamma}), 0.01)$-robust classifier can be learned from $O(d)$ samples from $D_0$ and $D_1$;
\item For $\widetilde{D_0}$ and $\widetilde{D_1}$, there exists a linear (non-robust) classifier, which can be learned in polynomial time;
\item For every $\eps > \rho$, in order to learn a $(\eps, 0.01)$-robust classifier for $\widetilde{D}_0$ and $\widetilde{D}_1$,
one needs at least $2^{d^{\Omega(1)}}$ statistical queries with accuracy as good as $2^{-d^{\Omega(\gamma)}}$.
\end{itemize}
\end{theorem}

For instance, if $\gamma$ is a small constant we get the existence of a $C$-robust classifier,
where $C$ is a large constant. One could push $C$ as high as $\Omega(\log^{1/2 - \eps} d)$
at a cost of the lower bound being against SQ queries with somewhat worse accuracy ($2^{-2^{\log^{\Omega(\eps)} d}}$ instead
of $2^{-d^{\Omega(1)}}$).

We first show a family of pairs $(D_0, D_1)$ that admit a robust classifier, yet it is hard (in the SQ model) to learn \emph{any} (non-robust) classifier.
Later, in Section~\ref{add_bogus_coordinate}, we show a simple modification of this family to obtain the main result.

\subsection{Hard family of distributions}

Here we define a hard family of pairs of distributions $(D_0, D_1)$ as discussed above.  This section contains the definition and key properties of the family; proofs of those properties appear in Appendix~\ref{app:distribution}.  This family can be seen to be a high-dimensional generalization and modification of a family considered in~\cite{DKS17}. The family depends on three parameters: integers $1 \leq k \leq d$, $m \geq 1$ and a positive real $\eps > 0$.

Fix an integer $m \geq 1$. We introduce two auxiliary distributions over $\Rbb$ that we will use later
as building blocks.

\define{one_dimensional_hard}{Lemma}{%
There exist two distributions $D_A$ and $D_B$ over $\Rbb$ with everywhere positive p.d.f.'s $A(t)$ and $B(t)$ respectively such that:
\begin{itemize}
\item $D_A$ and $D_B$ match $N(0, 1)$ in the first $m$ moments;
\item There exist two subsets $S_A, S_B \subset \Rbb$ such that the distance between $S_A$ and $S_B$ is at least $\Omega(1 / \sqrt{m})$,
$\mathbb{P}_{x \sim D_A}[x \in S_A] \geq 1 - e^{-\Omega(m)}$, and $\mathbb{P}_{x \sim D_B}[x \in S_B] \geq 1 - e^{-\Omega(m)}$;
\item $A, B \in C^{\infty}$, and for every $0 \leq l \leq m + 1$ and $t$, one has: $|\frac{d^l}{dt^l} \frac{A(t)}{G(t)}|, |\frac{d^l}{dt^l} \frac{B(t)}{G(t)}| \leq m^{O(l + 1)}$.
\end{itemize}

(See Figure~\ref{fig:hermite} for the illustration.)

}
\state{one_dimensional_hard}

\begin{figure}
  \centering
  \includegraphics[width=1.0\textwidth]{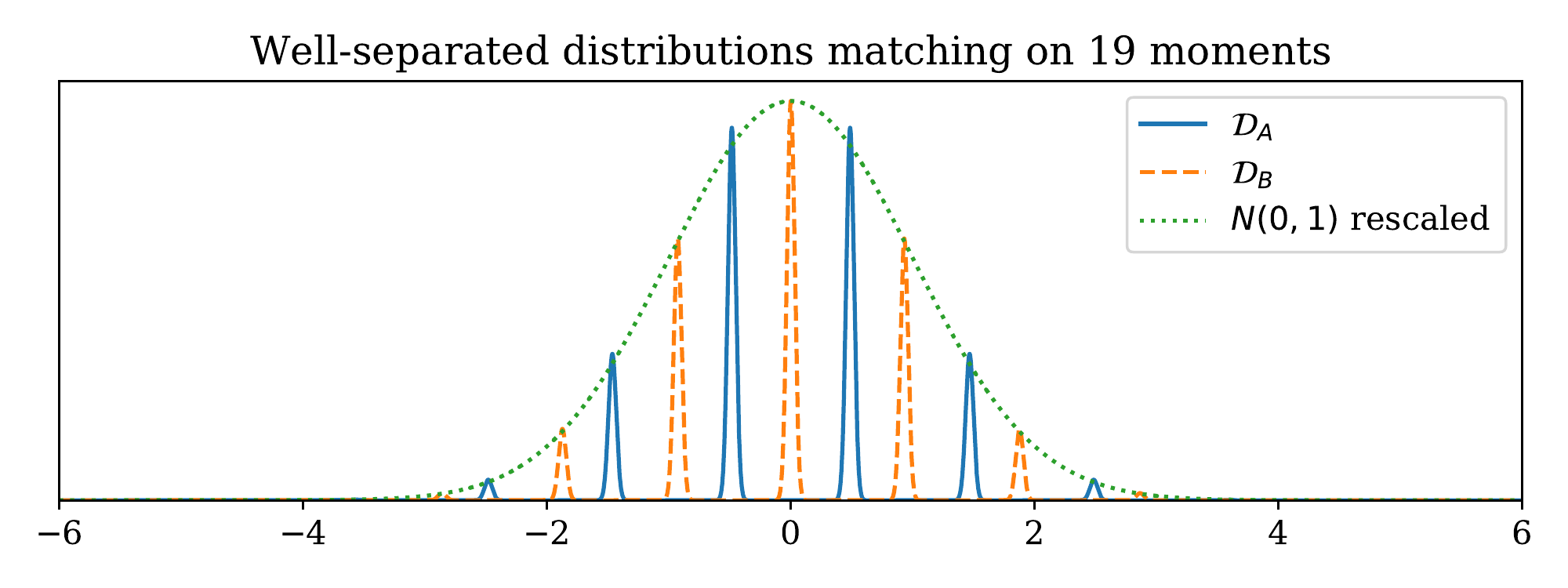}
  \caption{The distributions in Lemma~\ref{one_dimensional_hard} are
    similar to discretized Gaussians, with careful discretization and
    weighting from Gauss-Hermite quadrature.}
  \label{fig:hermite}
\end{figure}

Next let us fix parameters $1 \leq k \leq d$ and $\eps > 0$.
Let $\Uc = \{U_i\}$ be a family of $k$-dimensional subspaces of $\Rbb^d$
with fixed orthonormal bases such that for every $i \ne j$ and $u \in U_i$, one has: $\|\proj_{U_j} u\|_2 \leq \eps \cdot \|u\|_2$.
Informally speaking, subspaces from $\Uc$ are pairwise near-orthogonal.

\define{many_subspaces}{Lemma}{%
For every $k \leq d^{\Omega(1)}$,
there exists such a family $\Uc$ with $\eps \leq d^{-0.49}$ and $|\Uc| = 2^{d^{\Theta(1)}}$.
}
\state{many_subspaces}

Now we are ready to define our family of hard pairs $(D_0, D_1)$ of distributions over $\Rbb^d$.
The family is parameterized by a $k$-dimensional subspace $U \in \Uc$ together with an orthonormal basis $u_1, u_2, \ldots, u_k \in U$, where $\Uc$ is the family of
subspaces guaranteed by Lemma~\ref{many_subspaces}. Let us extend the above basis to a basis for the whole $\Rbb^d$: $u_1, u_2, \ldots, u_d$.
Now we define a pair of distributions $D_{U, A}$ and $D_{U, B}$ via their p.d.f.'s $A_U(x)$ and $B_U(x)$
respectively as follows:
$$
A_U(x) = \prod_{i=1}^k A(\langle x, u_i\rangle) \cdot \prod_{i=k+1}^d G(\langle x, u_i \rangle)
\qquad \text{and} \qquad
B_U(x) = \prod_{i=1}^k B(\langle x, u_i\rangle) \cdot \prod_{i=k+1}^d G(\langle x, u_i \rangle),
$$
where $A(\cdot)$ and $B(\cdot)$ are densities of distributions $D_A$
and $D_B$ from Lemma~\ref{one_dimensional_hard}, and
$G(t) = \frac{1}{\sqrt{2 \pi}} \cdot e^{-t^2 / 2}$ is the p.d.f.\ of
the standard Gaussian distribution $N(0, 1)$.  Now we simply take
$D_0$ to be $D_{U, A}$ and $D_1$ to be $D_{U, B}$.

\define{high_dimensional_separation}{Lemma}{%
There exist two sets $S_{U, A}, S_{U, B} \subset \Rbb^d$ such that the distance between $S_{U, A}$ and $S_{U, B}$ is $\Omega(\sqrt{k / m})$,
and for which $\mathbb{P}_{x \sim D_{U, A}}[x \in S_{U, A}] \geq 1 - e^{- \Omega(km)}$  and
 $\mathbb{P}_{x \sim D_{U, B}}[x \in S_{U, B}] \geq 1 - e^{- \Omega(km)}$.
}
\state{high_dimensional_separation}

As a result, the pair $(D_0, D_1)$ admits a $(\Omega(\sqrt{k / m}), e^{-km^{\Omega(1)}})$-robust classifier. Moreover, since $\log |\Uc| \leq O(d)$
(which follows from standard bounds on the number of pairwise near-orthogonal unit vectors in $\Rbb^d$),
it follows from Theorem~\ref{thm:finiteUB} that one can learn a $(\Omega(\sqrt{k / m}), 0.01)$-robust classifier from merely $O(d)$ samples.

\subsection[SQ lower bound for learning a classifier]{SQ lower bound for learning a classifier for $D_{U, A}$ and $D_{U, B}$}
\label{sq_facade}

The heart of the matter is to show that it requires $2^{d^{\Omega(1)}}$ statistical queries with precision $\tau = 2^{-d^{\Theta(\gamma)}}$ to learn a classifier
for $D_{U, A}$ and $D_{U, B}$ provided that all the parameters $m, k, \eps$ are set correctly.
The argument is fairly involved and uses the framework of~\cite{FGRVX} to reduce the question to that of upper bounding $\chi$-correlation between
the distributions. Due to space limitations, we show the argument in Appendix~\ref{sq_lb_appendix} of the supplementary material.

\subsection{Making the distribution easy to learn non-robustly}
\label{add_bogus_coordinate}

Let us now show a family of pairs distributions $(\widetilde{D}_0, \widetilde{D}_1)$ over $\Rbb^{d + 1}$ such that it is easy to learn a (non-robust) classifier, but hard to learn a robust one. The construction is very simple: we take distributions $(D_0, D_1)$ over $\Rbb^d$ as defined above and define $x \sim \widetilde{D}_0$
to be $x = (0, y_1, y_2, \ldots, y_d)$, where $y \sim D_0$, and, similarly,
$x \sim \widetilde{D}_1$ to be $x = (\rho, y_1, y_2, \ldots, y_d)$, where $y \sim D_1$ and $\rho > 0$.
These distributions admit a trivial (non-robust) classifier based on the first coordinate. Moreover,
since $\widetilde{D}_0$ and $\widetilde{D}_1$ are linearly separable, they can be classified using
linear SVM or logistic regression. Information-theoretically, one can learn a $(\sqrt{1 / \gamma}, 0.1)$-robust classifier using $O(d)$
samples by ignoring the first coordinate and applying Theorem~\ref{thm:finiteUB}.
However, for every $\eps > \rho$, one needs $2^{d^{\Omega(1)}}$ SQ queries with accuracy $2^{-d^{\Theta(\gamma)}}$
to learn an $(\eps, 0.1)$-robust separator. This can be shown exactly the same way as for $D_0$ and $D_1$ (see Appendix~\ref{sq_lb_appendix} in the supplementary material).

The above distributions are hard to learn robustly with respect to the
$\ell_2$ norm.  We can switch to $\ell_\infty$ by replacing $x$ by
its Hadamard transform $Hx$.  Since
$\norm{Hx - Hy}_{\infty} \geq \norm{H(x-y)}_2/\sqrt{d} =
\norm{x-y}_2$, the robustness parameters in the theorem are
unchanged while the diameter becomes $O(\sqrt{d \log d})$.

\section{Conclusion and future directions}
In this paper we put forward the thesis that adversarial examples might be an unavoidable consequence of computational constraints for learning algorithms. Our main piece of evidence is a classification task, for which there essentially exists a classifier robust to Euclidean perturbations of size $\log^{1/2 - \eps} d$ (while with high probability any sample has norm $O(\sqrt{d})$), yet finding {\em any} non-trivial robust classifier (even for arbitrarily small perturbations, and with probability of correctness only slightly better than chance) is hard in the statistical query model (in the sense that one needs an exponential number of queries, even with a very high precision statistical query oracle). We identify several directions in which this result could be strengthened to give stronger evidence for our thesis.
\begin{enumerate}
\item The most important question for the validity of our thesis is whether one could prove a similar hardness result for {\em natural} distributions. This is a particularly challenging open problem as the concept of a natural distribution is fuzzy (for instance there is no consensus on what a natural distribution for images should look like).
\item We believe that our proposed classification task is really computationally hard in any sense, not only in the statistical query model. As we discussed SQ is natural for learning theory hardness, but there have been lots of works leveraging other types of hardness assumption (e.g., cryptographic). It would be interesting to explore further the position of robust learning in the hardness landscape.
\item Finally one might wonder whether the perturbation size $\log^{1/2 - \eps} d$ is optimal (for distributions essentially supported in a ball of size $\sqrt{d}$). A concrete open question could be phrased as follows: consider a classification task that is $(\Psi(d), 0)$-robustly feasible, how fast does $\Psi$ need to grow in order to ensure that one can find in polynomial time a $(1, 1/3)$-robust classifier?
\end{enumerate}
\newpage
\bibliographystyle{plainnat}
\bibliography{bib}

\begin{thebibliography}{20}
\providecommand{\natexlab}[1]{#1}
\providecommand{\url}[1]{\texttt{#1}}
\expandafter\ifx\csname urlstyle\endcsname\relax
  \providecommand{\doi}[1]{doi: #1}\else
  \providecommand{\doi}{doi: \begingroup \urlstyle{rm}\Url}\fi

\bibitem[Athalye et~al.(2018)Athalye, Carlini, and Wagner]{ACW18}
Anish Athalye, Nicholas Carlini, and David Wagner.
\newblock Obfuscated gradients give a false sense of security: Circumventing
  defenses to adversarial examples.
\newblock In \emph{Proceedings of the 35th International Conference on Machine
  Learning}, ICML '18, 2018.
\newblock URL \url{https://arxiv.org/abs/1802.00420}.

\bibitem[Blum et~al.(1994)Blum, Furst, Jackson, Kearns, Mansour, and
  Rudich]{Blum94}
Avrim Blum, Merrick Furst, Jeffrey Jackson, Michael Kearns, Yishay Mansour, and
  Steven Rudich.
\newblock Weakly learning dnf and characterizing statistical query learning
  using fourier analysis.
\newblock In \emph{Proceedings of the Twenty-sixth Annual ACM Symposium on
  Theory of Computing}, STOC '94, pages 253--262. ACM, 1994.

\bibitem[Dalvi et~al.(2004)Dalvi, Domingos, Mausam, Sanghai, and
  Verma]{Dalvi04}
Nilesh Dalvi, Pedro Domingos, Mausam, Sumit Sanghai, and Deepak Verma.
\newblock Adversarial classification.
\newblock In \emph{Proceedings of the Tenth ACM SIGKDD International Conference
  on Knowledge Discovery and Data Mining}, KDD '04, pages 99--108. ACM, 2004.

\bibitem[Diakonikolas et~al.(2017)Diakonikolas, Kane, and Stewart]{DKS17}
Ilias Diakonikolas, Daniel Kane, and Alistair Stewart.
\newblock Statistical query lower bounds for robust estimation of
  high-dimensional gaussians and gaussian mixtures.
\newblock In \emph{Proceedings of the fifty-eighth Annual Symposium on
  Foundations of Computer Science}, FOCS '17, 2017.
\newblock URL \url{https://arxiv.org/pdf/1611.03473.pdf}.

\bibitem[Elsayed et~al.(2018)Elsayed, Shankar, Cheung, Papernot, Kurakin,
  Goodfellow, and Sohl-Dickstein]{elsayed2018adversarial}
Gamaleldin~F Elsayed, Shreya Shankar, Brian Cheung, Nicolas Papernot, Alex
  Kurakin, Ian Goodfellow, and Jascha Sohl-Dickstein.
\newblock Adversarial examples that fool both human and computer vision.
\newblock \emph{arXiv preprint arXiv:1802.08195}, 2018.

\bibitem[Fawzi et~al.(2018)Fawzi, Fawzi, and Fawzi]{FFF18}
Alhussein Fawzi, Hamza Fawzi, and Omar Fawzi.
\newblock Adversarial vulnerability for any classifier, 2018.
\newblock URL \url{https://arxiv.org/pdf/arXiv:1802.08686.pdf}.

\bibitem[Feldman(2017)]{feldman2017general}
Vitaly Feldman.
\newblock A general characterization of the statistical query complexity.
\newblock \emph{Proceedings of Machine Learning Research vol}, 65:\penalty0
  1--46, 2017.

\bibitem[Feldman et~al.(2013)Feldman, Grigorescu, Reyzin, Vempala, and
  Xiao]{FGRVX}
Vitaly Feldman, Elena Grigorescu, Lev Reyzin, Santosh Vempala, and Ying Xiao.
\newblock Statistical algorithms and a lower bound for detecting planted
  cliques.
\newblock In \emph{Proceedings of the forty-fifth annual ACM symposium on
  Theory of computing}, STOC '13, pages 655--664. ACM, 2013.

\bibitem[Gil et~al.(2018)Gil, Segura, and Temme]{gil2018asymptotic}
Amparo Gil, Javier Segura, and Nico~M Temme.
\newblock Asymptotic approximations to the nodes and weights of gauss--hermite
  and gauss--laguerre quadratures.
\newblock \emph{Studies in Applied Mathematics}, 140\penalty0 (3):\penalty0
  298--332, 2018.

\bibitem[Gilmer et~al.(2018)Gilmer, Metz, Faghri, Schoenholz, Raghu,
  Wattenberg, and Goodfellow]{G18}
Justin Gilmer, Luke Metz, Fartash Faghri, Sam Schoenholz, Maithra Raghu, Martin
  Wattenberg, and Ian Goodfellow.
\newblock Adversarial spheres.
\newblock In \emph{International Conference on Learning Representations
  Workshop}, 2018.
\newblock URL \url{https://arxiv.org/pdf/1801.02774.pdf}.

\bibitem[Globerson and Roweis(2006)]{GR06}
Amir Globerson and Sam Roweis.
\newblock Nightmare at test time: Robust learning by feature deletion.
\newblock In \emph{Proceedings of the 23rd International Conference on Machine
  Learning}, ICML '06, pages 353--360. ACM, 2006.

\bibitem[Kearns(1998)]{Kearns98}
Michael Kearns.
\newblock Efficient noise-tolerant learning from statistical queries.
\newblock \emph{Journal of the ACM (JACM)}, 45\penalty0 (6):\penalty0
  983--1006, 1998.

\bibitem[Klivans and Sherstov(2007)]{KS}
Adam~R. Klivans and Alexander~A. Sherstov.
\newblock Unconditional lower bounds for learning intersections of halfspaces.
\newblock \emph{Machine Learning}, 69\penalty0 (2):\penalty0 97--114, 2007.

\bibitem[Madry et~al.(2018)Madry, Makelov, Schmidt, Tsipras, and
  Vladu]{Madry18}
Aleksander Madry, Aleksandar Makelov, Ludwig Schmidt, Dimitris Tsipras, and
  Adrian Vladu.
\newblock Towards deep learning models resistant to adversarial attacks.
\newblock In \emph{International Conference on Learning Representations}, 2018.
\newblock URL \url{https://arxiv.org/pdf/1706.06083.pdf}.

\bibitem[Nesterov(2004)]{Nes04}
Y.~Nesterov.
\newblock \emph{Introductory lectures on convex optimization: A basic course}.
\newblock Kluwer Academic Publishers, 2004.

\bibitem[Schmidt et~al.(2018)Schmidt, Santurkar, Tsipras, Talwar, and
  Madry]{M18}
Ludwig Schmidt, Shibani Santurkar, Dimitris Tsipras, Kunal Talwar, and
  Aleksander Madry.
\newblock Adversarially robust generalization requires more data, 2018.
\newblock URL \url{https://arxiv.org/pdf/arXiv:1804.11285.pdf}.

\bibitem[Song et~al.(2017)Song, Vempala, Wilmes, and Xie]{song2017complexity}
Le~Song, Santosh Vempala, John Wilmes, and Bo~Xie.
\newblock On the complexity of learning neural networks.
\newblock In \emph{Advances in Neural Information Processing Systems}, pages
  5520--5528, 2017.

\bibitem[Szegedy et~al.(2013)Szegedy, Zaremba, Sutskever, Bruna, Erhan,
  Goodfellow, and Fergus]{G14}
Christian Szegedy, Wojciech Zaremba, Ilya Sutskever, Joan Bruna, Dumitru Erhan,
  Ian Goodfellow, and Rob Fergus.
\newblock Intriguing properties of neural networks.
\newblock In \emph{International Conference on Learning Representations}, 2013.
\newblock URL \url{https://arxiv.org/pdf/1312.6199.pdf}.

\bibitem[Szego(1939)]{szeg1939orthogonal}
Gabor Szego.
\newblock \emph{Orthogonal polynomials}, volume~23.
\newblock American Mathematical Soc., 1939.

\bibitem[Wang et~al.(2017)Wang, Jha, and Chaudhuri]{wang2017analyzing}
Yizhen Wang, Somesh Jha, and Kamalika Chaudhuri.
\newblock Analyzing the robustness of nearest neighbors to adversarial
  examples.
\newblock \emph{arXiv preprint arXiv:1706.03922}, 2017.

\end{thebibliography}
\newpage
\appendix
\section{Proofs of properties of the SQ hard distribution}\label{app:distribution}

We start with the following lemma on Hermite polynomials:

\begin{lemma}
\label{hermite_spacing}
For every $k > 1$, the distance between any roots of $H_{k-1}(t)$ and $H_k(t)$ is at least $\Omega(1 / \sqrt{k})$.
\end{lemma}
\begin{proof}
  It is known that extrema of $H_k$ are exactly zeros of
  $H_{k-1}$, which follows from $H_k' = 2k H_{k-1}$ and a lack of double roots. Thus, it is enough to show
  that extrema and zeros of $H_k$ are
  $\Omega(1 / \sqrt{k})$-separated.

Consider the case where $0 \leq u < v < w$ are such that $H_k(u) = H_k(w) = 0$, $H_k$ is positive between $u$ and $w$, and $H_k'(v) = 0$.
Let us show how to lower bound $v - u$.
Denote $F_k(t) = e^{-t^2/ 2} H_k(t)$. Clearly, $F_k(u) = F_k(w) = 0$ and $F_k$ is positive between $u$ and $w$ with a unique local maximum on $[u, w]$, which we denote by $v'$.
It is not hard to check that $v' \leq v$. Thus, it is enough to lower bound $v' - u$.
It is known (see, e.g.,~\cite[Section 5.5]{szeg1939orthogonal} that $F_k$ satisfies the ODE $Z'' + (2k + 1 - t^2) Z = 0$.
By comparing with $Z'' + (2k + 1) Z = 0$, we can get that lower bound $v - u \geq v' - u \geq \frac{\pi}{2 \sqrt{2k + 1}} = \Omega(1 / \sqrt{k})$.

Now let us lower bound $w - v$. It is known~\cite[Section 5.5]{szeg1939orthogonal} that $H_k$ satisfies the ODE $Z'' - 2 t Z' + 2k Z = 0$. By comparing this ODE with $Z'' - 2 w Z' + 2 k Z = 0$,
we get that $w - v \geq \frac{\arctan\left(\frac{\sqrt{2k - w^2}}{w}\right)}{\sqrt{2k - w^2}} \geq \Omega(1 / \sqrt{k})$. The latter step is due to $w \leq \sqrt{2 k}$ and that
the lower bound on $w - v$ is nonincreasing in $w$.

Other cases can be treated similarly.
\end{proof}

\restate{one_dimensional_hard}
\begin{proof}
Let $H_m(t)$ and $H_{m+1}(t)$ be two consecutive (physicist's) Hermite's polynomials.
It is a classic result in Gaussian quadrature (see, e.g.,~\cite{szeg1939orthogonal}) that for every $k$, there exists a discrete distribution supported on the zeros of $H_k(t / \sqrt{2})$,
which matches $N(0, 1)$ in the first $2k - 1$ moments. Let $\widetilde{D}_A$ denote such a distribution for $H_m$
and $\widetilde{D}_B$ the same for $H_{m+1}$. By Lemma~\ref{hermite_spacing}, the distance between the supports of
$\widetilde{D}_A$ and $\widetilde{D}_B$ is at least $\Omega(1 / \sqrt{m})$ and they both match $N(0, 1)$
in the first $2m - 1 \geq m$ moments.

Now, we obtain the desired distributions $D_A$ and $D_B$ as follows. Fix a small $\delta > 0$. The distribution $D_A$ is
defined as $\sqrt{1 - \delta} \cdot x + \sqrt{\delta} \cdot y$, where $x \sim \widetilde{D}_A$, $y \sim N(0, 1)$, and $x$ and $y$ are independent.
The distribution $D_B$ is defined similarly, but instead of $\widetilde{D}_A$ we use $\widetilde{D}_B$.
It is easy to check that $D_A$ and $D_B$ match the first $m$ moments of $N(0, 1)$.
Now suppose that $\delta = 1 / m^2$. The second property follows from the supports of $\widetilde{D}_A$
and $\widetilde{D}_B$ being $\Omega(1 / \sqrt{m})$ separated and the standard concentration inequalities; specifically,
we take $S_A$ to be the Minkowski sum of the support of scaled down $\widetilde{D}_A$ and the ball of radius $\Theta(1 / \sqrt{m})$,
and $S_B$ to be similar with $\widetilde{D}_B$ instead of $\widetilde{D}_A$. Then the chance $x \sim D_A$ is not in $S_A$ is at most the chance $y \sim N(0,1)$ has $|\sqrt{\delta} y| > \Omega(1/\sqrt{m})$, which is $e^{-\Omega(m)}$.

Now let us prove the bounds on $\frac{d^l}{dt^l} \frac{A(t)}{G(t)}$, for the $B(\cdot)$ similar bounds follows exactly the same way.

Denote $x_1 < x_2 < \ldots < x_m$ the roots of $H_m(t)$.

One has:
$$
A(t) = \frac{1}{\sqrt{2 \pi \delta}} \sum_{i=1}^{m} p_i e^{-\frac{\left(t - \sqrt{2 (1 - \delta)} x_i\right)^2}{2 \delta}},
$$
where $p_i = \mathbb{P}_{x \sim \widetilde{D}_A}[x = \sqrt{2} x_i]$.
Hence,
\begin{align*}
\frac{A(t)}{G(t)} &= \frac{1}{\sqrt{\delta}} \sum_{i=1}^{m} p_i e^{-\frac{\left(t - \sqrt{2 (1 - \delta)} x_i\right)^2}{2 \delta} + \frac{t^2}{2}}
\\& = \frac{1}{\sqrt{\delta}} \sum_{i=1}^m p_i e^{- \frac{1}{2} \cdot \left(\left(t \cdot \sqrt{\frac{1}{\delta} - 1} - \sqrt{\frac{2}{\delta}} \cdot x_i\right)^2 - 2x_i^2\right)}.
\\& = \frac{1}{\sqrt{\delta}} \sum_{i=1}^m p_i e^{- \frac{1}{2} \cdot \frac{\left(t \cdot \sqrt{1 - \delta} - \sqrt{2}x_i\right)^2}{\delta} + x_i^2}.
\end{align*}

We have for every $i$ the bound $p_i e^{x_i^2} = O(1)$~\cite{gil2018asymptotic}. Therefore, if $Q(t)$ denotes the p.d.f.\ of $N(0, \delta/(1-\delta))$ we have
\[
  \sup_t \left| \frac{d^l}{dt^l} \frac{A(t)}{G(t)}\right| \leq \frac{1}{\sqrt{\delta}} \cdot m \cdot O(1) \cdot \left(\sup_t \left|\frac{d^l}{dt^l} Q(t)\right|\right) = m(l/\delta)^{O(l+1)} = m^{O(l + 1)}.
\]
\end{proof}

\restate{many_subspaces}
\begin{proof}
Let $U$ and $V$ be uniformly random $k$-dimensional subspaces of $\Rbb^d$.
W.l.o.g. we can assume that $U$ is spanned by the first $k$ standard basis vectors.
Let $V$ be spanned by an orthonormal basis $v_1, v_2, \ldots, v_k$ such that each $v_i$
is distributed uniformly on the unit sphere of $\Rbb^d$.
Consider an $\eps'$-net $\mathcal{N}$ of the unit sphere of $U$ of size $(1 / \eps')^{O(k)}$.
For every $u \in \mathcal{N}$ with probability at least $1 - e^{-\Omega(\eps'^2 d)}$
the absolute value of the dot product of $u$ with a given $v_i$ is at most $\eps'$. As a result,
with probability at least $1 - (1 / \eps')^{O(k)} e^{-\Omega(\eps'^2 d)}$,
dot products between all elements of $\mathcal{N}$ and all $v_i$ are at most $\eps'$ in the absolute value.
But this implies that the dot products between all the unit vectors of $U$ and $V$ are at most $\eps' \sqrt{k}$.
So, by setting $\eps' = \eps / \sqrt{k}$ and by using the union bound, we get that we can set:
$$
\log |\Uc| \leq \Omega(\eps^2 d / k) - O(k (\log k + \log (1 / \eps))).
$$
Thus, we can set $\eps = d^{-0.49}$, and $k \leq d^{\sigma}$ for a sufficiently small positive $\sigma$,
which yields $|\Uc| = 2^{d^{\Theta(1)}}$.
\end{proof}

\restate{high_dimensional_separation}
\begin{proof}
The sets are defined as follows:
\begin{equation*}
S_{U, A} = \{x \in \Rbb^d \mid \mbox{for at least $0.9$-fraction of $1 \leq i \leq k$, one has $\langle x, u_i\rangle \in S_A$}\}
\end{equation*}
and
\begin{equation*}
S_{U, B} = \{x \in \Rbb^d \mid \mbox{for at least $0.9$-fraction of $1 \leq i \leq k$, one has $\langle x, u_i\rangle \in S_B$}\}.
\end{equation*}
The points $x \in S_{U, A}$ and $y \in S_{U, B}$ are well-separated, since in at least a $0.8$-fraction of $1 \leq i \leq k$, both $\langle x, u_i \rangle \in S_A$
and $\langle y, u_i \rangle \in S_B$. Since $S_A$ and $S_B$ are $\Omega(1 / \sqrt{m})$-separated, we obtain the result.

The bounds on the probabilities follow from the respective bounds in Lemma~\ref{one_dimensional_hard} and standard Chernoff bounds.
\end{proof}

\section{SQ lower bound}
\label{sq_lb_appendix}

\subsection{SQ lower bound}

Now let us show that if we set all the parameters appropriately, it is hard in the SQ model to learn a good classifier (robust or otherwise)
for distributions $D_{U, A}$ and $D_{U, B}$ defined above, where $U \in \Uc$ is an unknown subspace.
The main idea is to show that if the subspace $U \in \Uc$ is chosen uniformly at random,
unless we perform more than $2^{d^{\Omega(1)}}$ queries, we can not tell apart
$D_{U, A}$ or $D_{U, B}$ from the standard Gaussian $N(0, I_d)$ (and as a result, from each other).
Intuitively, any since query can only reliably distinguish $D_{U, A}$ from $N(0, I_d)$ for a tiny fraction of subspaces $U \in \Uc$. The result then follows by a simple counting argument.
To formalize the above intuition, we use an argument similar at a high-level to the one used in~\cite{DKS17}.

Let $D, D_1, D_2$ be distributions over $\Rbb^d$ with everywhere positive p.d.f.'s $P(x)$, $P_1(x)$,
and $P_2(x)$, respectively. Then, the pairwise correlation of $D_1$ and $D_2$ w.r.t.\ $D$, denoted
by $\chi_D(D_1, D_2)$, is defined as follows:
$$
\chi_D(D_1, D_2) = \int_{\Rbb^d} \frac{P_1(x) P_2(x)}{P(x)} \, dx - 1.
$$

In Section~\ref{correlation_section}, we show that for an appropriate setting of parameters
(namely, when $\eps m^{\Theta(1)} k \leq d^{-\Omega(1)}$), for every $U_1, U_2 \in \Uc$,
one has:
$$
\chi_{N(0, I_d)}(D_{U_1, A}, D_{U_2, A}) \leq
\begin{cases}
m^{O(k)} & \mbox{if $U_1 = U_2$}\\
m^{O(k)} \cdot d^{-\Omega(m)} & \mbox{otherwise}
\end{cases}
$$
and
$$
\chi_{N(0, I_d)}(D_{U_1, B}, D_{U_2, B}) \leq
\begin{cases}
m^{O(k)} & \mbox{if $U_1 = U_2$}\\
m^{O(k)} \cdot d^{-\Omega(m)} & \mbox{otherwise}.
\end{cases}
$$

Then by repeating the proof of Lemma~3.3 from~\cite{FGRVX}, we get that if the number of queries is significantly smaller
than:
$$
\frac{|\Uc| \cdot (\tau^2 - m^{O(k)} d^{-\Omega(m)})}{m^{O(k)}},
$$
then with high probability over a random subspace $U \in \Uc$, all the queries asked can be answered as if both $D_{U, A}$ and $D_{U, B}$ were $N(0, I_d)$.
As a result, we cannot distinguish them from $N(0, I_d)$ and, as a result, between each other.

Suppose that $m \log d > C k \log m$ for a sufficiently large constant
$C$, so that the $m^{O(k)}d^{-\Omega(m)}$ term is less than
$d^{-\Omega(m)} < m^{-\Omega(k)}$. Then we can set the precision
$\tau$ to $m^{-\Theta(k)}$ and still be unable to distinguish
$D_{U,A}$ from $D_{U,B}$ from
$|\Uc| m^{-O(k)} = 2^{d^{\Omega(1)}}m^{-O(k)}$ queries. If
$m^{O(k)} \leq 2^{d^{\sigma}}$ for a sufficiently small positive
$\sigma > 0$, this gives the desired lower bound of  $2^{d^{\Omega(1)}}$ on the
number of SQ queries the algorithm must ask.

\subsection{Upper bounding pairwise correlations}
\label{correlation_section}

In this section, we show how to upper bound
$\chi_{N(0, I_d)}(D_{U_1, A}, D_{U_2, A})$; upper bounding
$\chi_{N(0, I_d)}(D_{U_1, B}, D_{U_2, B})$ is exactly the same.
Denote $a(t) = \frac{A(t)}{G(t)} - 1$, where $G(t)$ is the p.d.f.\ of
a standard Gaussian. By Lemma~\ref{one_dimensional_hard}, one has
$\mathbb{E}_{t \sim N(0, 1)}[t^l \cdot a(t)] = 0$ for all $l \in \{1, 2, \dotsc, m\}$.

We assume that $m^C \eps k \leq d^{-\Omega(1)}$ for a sufficiently large constant $C$ to be determined later. Since by Lemma~\ref{many_subspaces} we can take $\eps=d^{-0.49}$, the required inequality
holds as long as $m$ and $k$ are at most small powers of $d$.

First, suppose that $U_1 = U_2 = U$. Suppose that $u_1, u_2, \ldots, u_d$
is an orthonormal basis of $\Rbb^d$ such that $u_1, u_2, \ldots, u_k$
is a fixed basis of $U$. Then,
\begin{align*}
\chi_{N(0, 1)^d}(D_{U, A}, D_{U, A})
& = \int_{\Rbb^d} \frac{A_U(x)^2}{\prod_{i=1}^d G(\langle x, u_i \rangle)} \, dx - 1\\
& = \int_{\Rbb^d} \prod_{i=1}^k (1 + a(\langle x, u_i \rangle))^2 \cdot \prod_{i=1}^d G(\langle x, u_i \rangle) \, dx - 1\\
& = \mathbb{E}_{x \sim N(0, I_d)}\left[\prod_{i=1}^k (1 + a(\langle x, u_i \rangle))^2\right] - 1\\
& = \prod_{i=1}^k \mathbb{E}_{x \sim N(0, I_d)}\left[(1 + a(\langle x, u_i \rangle))^2\right] - 1\\
& \leq m^{O(k)},
\end{align*}
where the fourth step is due to the independence of $\langle x, u_i \rangle$ (which is implied by orthogonality of $u_i$),
and the fifth step follows from Lemma~\ref{one_dimensional_hard}.

Now suppose that $U_1 \ne U_2$. Suppose that $u_1, u_2, \ldots, u_d$
is an orthonormal basis of $\Rbb^d$ such that $u_1, u_2, \ldots, u_k$
is a fixed basis of $U_1$, and, similarly, $v_1, v_2, \ldots, v_d$
is an orthonormal basis of $\Rbb^d$ such that $v_1, v_2, \ldots, v_k$
is a fixed basis of $U_2$. Now,
\begin{align}
\label{chi_expansion}
\chi_{N(0, I_d)}(D_{U_1, A}, D_{U_2, A}) &= \int \frac{A_{U_1}(x) A_{U_2}(x)}{\prod_{i=1}^d G(x_i)} \, dx - 1\nonumber\\
& = \mathbb{E}_{x \sim N(0, I_d)}\left[\prod_{i=1}^k \Bigl(1 + a(\langle x, u_i\rangle)\Bigr)\cdot \prod_{i=1}^k \Bigl(1 + a(\langle x, v_i\rangle)\Bigr)\right] - 1\nonumber\\
&=\sum_{S, T \subseteq [k]} \mathbb{E}_{x \sim N(0, I_d)}\left[\prod_{i\in S} a(\langle x, u_i\rangle)\cdot \prod_{i\in T} a(\langle x, v_i\rangle)\right] - 1\nonumber\\
&=\sum_{\begin{smallmatrix}S, T \subseteq [k]:\\S, T \ne \emptyset\end{smallmatrix}} \mathbb{E}_{x \sim N(0, I_d)}\left[\prod_{i\in S} a(\langle x, u_i\rangle)\cdot \prod_{i\in T} a(\langle x, v_i\rangle)\right],
\end{align}
where the last step follows from the fact that if $S = \emptyset$ and $T \ne \emptyset$, then the expression factorizes due to the independence
of $\langle x, v_i \rangle$, and we also use that $\mathbb{E}_{t \sim N(0, 1)}[a(t)] = 0$. The case $S \ne \emptyset$ and $T = \emptyset$ is similar.

Now let us fix non-empty $S, T \subseteq [k]$. W.l.o.g., suppose that $|S| \geq |T|$. Denote $\widetilde{v}_i = v_i - \mathrm{proj}_{U_1} v_i$.
Since $U_1, U_2 \in \Uc$ and $U_1 \ne U_2$, we have $\|\widetilde{v}_i - v_i\|_2 \leq \eps$.
One has for every
$1 \leq i \leq k$ by a Taylor expansion that
\begin{equation}
\label{taylor_expansion}
a(\langle x, v_i\rangle) = \sum_{l=0}^m a^{(l)}(\langle x, \widetilde{v}_i\rangle) \cdot \frac{\langle x, v_i - \widetilde{v}_i\rangle^l}{l!} + a^{(m+1)}(\theta_i) \cdot \frac{\langle x, v_i - \widetilde{v}_i\rangle^{m+1}}{(m+1)!},
\end{equation}
for some $\theta_i = \theta_i(x)$ that lies between $\langle x, \wt{v}_i\rangle$ and $\langle x, v_i\rangle$.

\begin{lemma}
\label{coeff_killing}
Suppose $|S| \geq |T|$.  For every $l \colon T \to \{0, 1, \ldots, m\}$, one has:
$$
\mathbb{E}_{x \sim N(0, I_d)}\left[\prod_{i \in S} a(\langle x, u_i\rangle)\cdot \prod_{i \in T} \left(a^{(l(i))}(\langle x, \widetilde{v}_i\rangle)\cdot \frac{\langle x, v_i - \widetilde{v}_i\rangle^{l(i)}}{l(i)!}\right)\right] = 0.
$$
\end{lemma}
\begin{proof}
Since $v_i - \widetilde{v}_i \in U_1$, we can write $v_i - \widetilde{v}_i = \sum_{j=1}^k \alpha_{ij} u_j$.
One has:
\begin{align*}
& \mathbb{E}_{x \sim N(0, I_d)}\left[\prod_{i \in S} a(\langle x, u_i\rangle)\cdot \prod_{i \in T} \left(a^{(l(i))}(\langle x, \widetilde{v}_i\rangle)\cdot \frac{\langle x, v_i - \widetilde{v}_i\rangle^{l(i)}}{l(i)!}\right)\right]\\
& = \mathbb{E}_{x \sim N(0, I_d)}\left[\prod_{i \in S} a(\langle x, u_i\rangle)\cdot \prod_{i \in T} \left(a^{(l(i))}(\langle x, \widetilde{v}_i\rangle)\cdot \frac{\left(\sum_{j=1}^k \alpha_{ij} \langle x, u_j \rangle \right)^{l(i)}}{l(i)!}\right)\right]\\
& = \mathbb{E}_{x \sim N(0, I_d)}\left[\prod_{i \in S} a(\langle x, u_i\rangle)\cdot \prod_{i \in T} \left(a^{(l(i))}(\langle x, \widetilde{v}_i\rangle)\cdot \frac{1}{l(i)!} \cdot \sum_{\beta_{ij}: \sum_{j=1}^k \beta_{ij} = l(i)} \binom{l(i)}{\beta_{i1} \ldots \beta_{ik}} \prod_{j=1}^k (\alpha_{ij} \langle x, u_j \rangle)^{\beta_{ij}}\right)\right]\\
& = \sum_{\substack{\beta_{ij}\\\forall i: \sum_{j=1}^k \beta_{ij} = l(i)}} \left(\prod_{i \in T} \frac{\binom{l(i)}{\beta_{i1} \ldots \beta_{ik}}}{l(i)!}\right)\cdot \mathbb{E}_{x \sim N(0, I_d)}\left[\prod_{i \in S} a(\langle x, u_i\rangle)\cdot \prod_{i \in T} \left(a^{(l(i))}(\langle x, \widetilde{v}_i\rangle)\cdot \prod_{j=1}^k (\alpha_{ij} \langle x, u_j \rangle)^{\beta_{ij}}\right)\right].\\
\end{align*}
Now let us fix partitions $\beta_{ij}$ and show that:
\begin{equation}
\label{large_zero}
\mathbb{E}_{x \sim N(0, I_d)}\left[\prod_{i \in S} a(\langle x, u_i\rangle)\cdot \prod_{i \in T} \left(a^{(l(i))}(\langle x, \widetilde{v}_i\rangle)\cdot \prod_{j=1}^k(\alpha_{ij} \langle x, u_j \rangle)^{\beta_{ij}}\right)\right] = 0.
\end{equation}
Since $\sum_{ij} \beta_{ij} = \sum_i l(i) \leq |T| \cdot m$, there exists $j^* \in S$ such that:
$\sum_i \beta_{ij^*} \leq \frac{|T| \cdot m}{|S|} \leq m$.
Since $\langle x, u_{j^*}\rangle$ is independent from the remaining dot products,
we can factor from~(\ref{large_zero}) the expression
\begin{equation}
\label{small_zero}
\mathbb{E}_{x \sim N(0, I_d)}[a(\langle x, u_{j^*}\rangle) \langle x, u_{j^*}\rangle^l]
\end{equation}
with $l \leq m$. But since $\langle x, u_{j^*} \rangle$ is distributed as $N(0, 1)$,
one has that~(\ref{small_zero}) is equal to zero due to Lemma~\ref{one_dimensional_hard}.
\end{proof}

Let us continue upper bounding~(\ref{chi_expansion}). For $i \in T$ and $0 \leq j \leq m + 1$, denote:
$$
\gamma_{ij} = \begin{cases}
v_i - \widetilde{v}_i, \mbox{if $j \leq m$},\\
\theta_i, \mbox{if $j = m + 1$}.
\end{cases}
$$
One has:
\begin{align}
\label{main_step}
&\mathbb{E}_{x \sim N(0, I_d)}\left[\prod_{i\in S} a(\langle x, u_i\rangle)\cdot \prod_{i\in T} a(\langle x, v_i\rangle)\right]
\nonumber\\&= \sum_{l \colon T \to \{0, 1, \ldots, m + 1\}} \mathbb{E}_{x \sim N(0, I_d)}\left[\prod_{i \in S} a(\langle x, u_i\rangle)\cdot
\prod_{i \in T} a^{(l(i))}(\gamma_{i,l(i)}) \cdot \frac{\langle x, v_i - \widetilde{v}_i \rangle^{l(i)}}{l(i)!}\right]
\nonumber\\&=
\sum_{\begin{smallmatrix}l \colon T \to \{0, 1, \ldots, m + 1\}\\ l^{-1}(m+1)\ne \emptyset\end{smallmatrix}} \mathbb{E}_{x \sim N(0, I_d)}\left[\prod_{i \in S} a(\langle x, u_i\rangle)\cdot\prod_{i \in T} a^{(l(i))}(\gamma_{i, l(i)}) \cdot \frac{\langle x, v_i - \widetilde{v}_i \rangle^{l(i)}}{l(i)!}\right]
\nonumber\\&\leq
\sum_{\begin{smallmatrix}l \colon T \to \{0, 1, \ldots, m + 1\}\\ l^{-1}(m+1)\ne \emptyset\end{smallmatrix}}
(\sup_{t} |a(t)|)^{|S|} \cdot \left(\prod_{i \in T} \frac{\sup_t |a^{l(i)}(t)|}{l(i)!}\right) \cdot \mathbb{E}_{x \sim N(0, I_d)}\left[\prod_{i \in T}\|\mathrm{proj}_{U_1} x\|_2^{l(i)} \cdot \|v_i - \widetilde{v}_i\|_2^{l(i)}\right]
\nonumber\\&\leq
\sum_{\begin{smallmatrix}l \colon T \to \{0, 1, \ldots, m + 1\}\\ l^{-1}(m+1)\ne \emptyset\end{smallmatrix}}
m^{O(|S|)} \cdot \left(\prod_{i \in T} \frac{m^{O(l(i))} \cdot \eps^{l(i)}}{l(i)!}\right) \cdot \mathbb{E}_{y \sim N(0, I_{k})} \left[\|y\|_2^{\sum_{i \in T} l(i)}\right]
\nonumber\\&\leq
\sum_{\begin{smallmatrix}l \colon T \to \{0, 1, \ldots, m + 1\}\\ l^{-1}(m+1)\ne \emptyset\end{smallmatrix}}
\frac{m^{O\left(|S| + \sum_{i\in T} l(i)\right)} \cdot \eps^{\sum_{i \in T} l(i)} \cdot \left(k + \sum_{i \in T} l(i)\right)^{\sum_{i \in T} l(i)}}{\prod_{i \in T} l(i)!}
\nonumber\\&\leq
\sum_{\begin{smallmatrix}l \colon T \to \{0, 1, \ldots, m + 1\}\\ l^{-1}(m+1)\ne \emptyset\end{smallmatrix}}
\frac{m^{O\left(k\right)}d^{-\Omega\left(\sum_{i\in T} l(i)\right)}}{\prod_{i \in T} l(i)!}
\nonumber\\&\leq m^{O(k)} \cdot d^{-\Omega(m)},
\end{align}
where the first step follows from~(\ref{taylor_expansion}), the second
step follows from Lemma~\ref{coeff_killing}, the third step follows
from Cauchy--Schwartz, the fourth step follows from
Lemma~\ref{one_dimensional_hard} and from the bound
$\|v_i - \widetilde{v}_i\|_2 \leq \eps$, the fifth step follows from
the inequality
$\mathbb{E}_{y \sim N(0, I_k)}[\|y\|_2^s] \leq (k + s)^s$, the sixth
step follows from $(\eps m^{\Theta(1)} k) = d^{-\Omega(1)}$ and from
$\sum_{i \in T} l(i) \leq O(mk)$, and the last step follows from
dropping the denominators, the sum having at most
$(m+2)^{|T|} = m^{O(k)}$ terms, and that $\sum_i l(i) \geq m + 1$.

Plugging~(\ref{main_step}) into~(\ref{chi_expansion}), we get the result.

\subsection{Setting parameters}

We obtain a $\Omega(\sqrt{k / m})$-robust classifier, and the precision of statistical queries can be as high as $m^{O(k)} \cdot d^{-\Omega(m)}$. Thus, for $0 < \gamma < 1/10$, we can set $m = d^{\Theta(\gamma)}$
and $k \ll \frac{m \log d}{\log m}$. As a result we get robustness
$\Omega(\sqrt{k / m}) = \Omega(\sqrt{\log d / \log m}) = \Omega(\sqrt{1 / \gamma})$,
and the precision of statistical queries can be as good as $2^{-d^{\Omega(\gamma)}}$.

\section{Bound on covering number of generative models}\label{app:gencover}
\begin{lemma}\label{lem:gwlipschitz}
  Let $g_w$ be a $\ell$-layer neural network architecture with at most
  $d$ activations in each layer and Lipschitz nonlinearities such as
  ReLUs.  Then
  \[
    \norm{g_w(x) - g_{w'}(x)}_2 \leq \norm{w - w'}_1 \cdot \norm{x}_2 \cdot (dB)^{\ell}
  \]
\end{lemma}
\begin{proof}
  By the triangle inequality, it suffices to consider $w$ and $w'$
  that differ in a single coordinate.  Suppose this coordinate is in
  layer $i$.  Since each layer's weight matrix $w_i$ has
  $\norm{w_i} \leq \norm{w_i}_F \leq dB$, and the initial layer has
  activation $\norm{x}_2$, the $\ell_2$ norm of the activations in the
  $i$th layer is at most $\norm{x}_2 (dB)^{i}$.  Therefore the change
  in activation in layer $i+1$ is at most
  $\norm{w - w'}_1 \cdot \norm{x}_2 (dB)^{i}$, and the change in the
  last layer is at most $\norm{w - w'}_1 \cdot \norm{x}_2 (dB)^{\ell-1}$.
\end{proof}

\restate{lem:gencover}
\begin{proof}
  First, consider any $w \in [-B, B]^m$ and $x \in \R^k$, and let $w'$
  differ from $w$ in a single weight.

  For some parameter $\alpha > 0$, we consider the net
  $\wt{\cN} = \{D(g_w) \mid w \in [-B, B]^m \cap \alpha
  \mathbb{Z}^m\}$. Our cover of $\D$ will be $\cN \times \cN$.  This
  has size $(1 + \frac{2B}{\alpha})^{2m}$, which is sufficiently small as
  long as $\alpha = (dB/\delta)^{-O(\ell)}$.

  It suffices to show for each $D \in \D$ and $i \in \{0, 1\}$ that
  $D_i \in U_{\eps+\delta, \delta}(\wt{D})$ for some
  $\wt{D} \in \wt{\cN}$.  Let $w^*$ be the $w$ for which
  $W_1(D_i, D(g_w)) \leq \eps$ and $\wh{w}$ be the nearest $w$ in our
  cover, so $\norm{\wh{w} - w^*}_\infty \leq \alpha$.  Then for any
  $x \in \R^k$ with $\norm{x}_2 \leq \sqrt{k}/\delta$,
  \[
    \norm{g_{\wh{w}}(x) - g_{w^*}(x)}_2 \leq \delta
  \]
  by Lemma~\ref{lem:gwlipschitz} and our chosen $\alpha$.  Since
  $\norm{x}_2 \leq \sqrt{k}/\delta$ with probability much higher than
  $1-\delta$, this implies
  $D(g_{w^*}) \in U_{\delta, \delta}(D(g_{\wh{w}}))$.  The triangle
  inequality then gives
  $D_i \in U_{\eps+\delta, \delta}(D(g_{\wh{w}}))$ as desired.
\end{proof}

\end{document}